\documentclass[12]{article} 
\usepackage{authblk}
\usepackage{url}
\usepackage{helvet} 
\usepackage{courier} 
\usepackage{graphicx} 
\usepackage{algorithm}
\usepackage{algorithmic}
\usepackage{amssymb} 
\usepackage{amsmath}
\usepackage{amsthm}
\usepackage{mathrsfs}
\usepackage{newfloat}
\usepackage{listings}
\floatstyle{ruled}
\newfloat{listing}{tb}{lst}{}
\floatname{listing}{Listing}
\newtheorem{theorem}{Theorem}

\newtheorem{prop}[theorem]{Proposition}
\newtheorem{definition}[theorem]{Definition}

\newtheorem{remark}[theorem]{Remark}

\title{\bf Finding Probably Approximate Optimal  Solutions by Training to Estimate the Optimal Values of Subproblems}
\author[*]{Nimrod Megiddo}
\author[**]{Segev Wasserkrug}
\author[**]{Orit Davidovich}
\author[***]{Simrit Shtern}
\affil[*]{IBM Research -- Almaden, San Jose, CA, USA}
\affil[**]{IBM Research -- Israel, Haifa, Israel}
\affil[***]{Faculty of Data and Decision Sciences\\
Technion, Haifa, Israel}

\date{October 28, 2025}

\newcommand{\ba}{\mbox{\boldmath $a$}}

\newcommand{\ssB}{{\cal B}}
\newcommand{\bc}{\mbox{\boldmath $c$}}

\newcommand{\be}{\mbox{\boldmath $e$}}
\newcommand{\besmall}{\mbox{\boldmath $\scriptstyle e$}}
\newcommand{\ssF}{{\cal F}}
\newcommand{\ssI}{{\cal I}}
\newcommand{\bi}{\mbox{\boldmath $i$}}
\newcommand{\bismall}{\mbox{\boldmath $\scriptstyle i$}}

\newcommand{\ssP}{{\cal P}}
\newcommand{\bR}{\mbox{\boldmath $R$}}

\newcommand{\bx}{\mbox{\boldmath $x$}}

\newcommand{\bw}{\mbox{\boldmath $w$}}
\newcommand{\btheta}{\mbox{\boldmath $\theta$}}
\newcommand{\bthetasmall}{\mbox{\boldmath $\scriptstyle \theta$}}
\newcommand{\bxi}{\mbox{\boldmath $\xi$}}
\newcommand{\bxismall}{\mbox{\boldmath $\scriptstyle \xi$}}
\newcommand{\betta}{\mbox{\boldmath $\eta$}}
\newcommand{\bettasmall}{\mbox{\boldmath $\scriptstyle \eta$}}
\newcommand{\bzero}{\mbox{\boldmath $0$}}
\newcommand{\bul}{\mbox{\small $\bullet$}}

\newcommand{\EE}{\mbox{\bf E}}
\newcommand{\bbN}{\mathbb{N}}
\newcommand{\bbR}{\mathbb{R}}

\newcommand{\sgn}{\mathop{\rm sgn}}
\def\mAth{\mathsurround=0pt}
\def\eqalign#1{\,\vcenter{\openup1\jot \mAth
\ialign{\strut\hfil$\displaystyle{##}$&$\displaystyle{{}##}$\hfil
    \crcr#1\crcr}}\,}

\newcommand{\ist}{{\rm({\it i}\,{\rm )}}}
\newcommand{\ind}{{\rm({\it ii}\,{\rm )}}}

\newcommand{\itemi}{\item[\ist]}
\newcommand{\itemii}{\item[\ind]}

\newcommand{\eq}{\begin{equation}}
\newcommand{\eeq}{\end{equation}}

\begin{document}
\maketitle
\begin{abstract}
The paper is about developing a solver for maximizing a real-valued function of binary variables.  
The solver relies on an algorithm that estimates the optimal objective-function value of instances from the underlying distribution of objectives and their respective sub-instances. 
The training of the estimator 
is based on an inequality that facilitates the use of the expected total deviation from optimality conditions as a loss function rather than the objective-function itself.
Thus, it does not calculate values of policies, nor does it rely on solved instances.
\end{abstract}

\section{\hskip -16pt. Introduction}
We consider a 
{\em discrete} (or {\em combinatorial}\/) {\em optimization problem} $\ssP$ as an enumerable collection 
$\ssI = \bigcup_{n=1}^\infty \ssI_n$ of {\em problem instances}, where $\ssI_n$ denotes the sub-collection of $\ssI$ of instances of dimension $n$, i.e., with $n$ discrete decision variables. For simplicity, we assume throughout that all the decision variables are binary. 
We assume a probability distribution
$p:\ssI \to [0,1]$ over $\ssI$,
i.e., 
$\sum_{\bismall\in\ssI} p(\bi)=1$
and $p(\bi)\ge 0$ for every $i\in \ssI$. 

In combinatorial optimization, an instance of dimension $n$ consists
of a feasible domain 
$F\subseteq B^n$ and an objective-function $\phi:F \to \bbR$. 
This problem instance calls for finding an $\bx \in F$ that maximizes $\phi(\bx)$.  
For our purposes here, it is easier to reformulate the problem as unconstrained. This can be done with the use of so-called artificial variables. 

In recent years, many papers have proposed to generate solutions of multi-variable problems by training 
neural networks to fix the values of the variables, sequentially, one at a time. 
See, for example, \cite{mnih2015human, kool2018attention, bresson2021transformer}.
The use of machine learning techniques for value approximations were first suggested in \cite{zhang1995value}.
A survey of machine learning for combinatorial optimization is given in \cite{bengio2021machine}.

In the process of sequential fixing, the network attempts to predict the best choice of value for every sequence of possible previous choices.
The method we propose here is different in the sense that it does not train to predict the best decision. Instead, it attempts to accurately estimate the optimal values of residual sub-problems arising during the fixing process, and the choices of variable values are simply driven by these estimates. 
The key point is that the training to estimate the optimal value does not require any previously-solved instances.
Instead, the training amounts to
iteratively minimizing the deviation from the conditions that characterize an optimal value function.  We prove an inequality that establishes a relation between the proximity to the true optimal value and the amount of deviation from the optimality conditions.
It is important to note that in this approach, the depth of network does not depend on the number of decision variables.

The hypothesis underlying this paper is as follows.  
For some problems $\ssP$, for some distributions $p$ over instances of $\ssP$, there exist neural-network models that can, with high probability under $p$, predict the optimal value of an instance of $\bi\in\ssP$ approximately.

A neural-network model typically has a large number of parameters that need to be tuned before the model reaches an acceptable level of performance.  The obvious idea of tuning the model using instances
with known optimal objective-function value is not practical.
We examine here a different idea, which does not use any optimal values.  Instead we rely on the discrepancy of the model predictions with respect to an optimality condition and the tuning is tantamount to minimizing that discrepancy.

\section{\hskip -16pt. Maximizing real-valued functions of binary  variables}
\begin{definition}\rm
A real-valued\footnote{For computation, we assume the function values are rational numbers.} function of $n$ binary variables
is a mapping
$ f:\{0,1\}^n \to \bbR$.
Given such a function $f$, 
the problem is to find
$x_1,\ldots,x_n \in \{0,1\}$ that maximize $f(x_1,\ldots,x_n)$.
\end{definition}

We assume a probability distribution over a denumerable set $\ssF$ of such functions $f$.

For $n\in\bbN$ and $k=1,\ldots,n$, denote
\[ \ssB_{n,k} = \big\{
\bxi\in \{0,1\}^n ~:~\xi_1=\cdots=
\xi_k=0\big\}~.\]
Also, for $\ell=1,\ldots,k$ and
$\betta\in\ssB_{n,k}$, denote
\[ \Xi_\ell(\betta) 
= \Xi_\ell(\betta;k)
= \big\{  \bxi \in \ssB_{n,\ell}:
\xi_j = \eta_j,~j=k+1,\ldots,n \big\} ~.\]
For every natural $n$, for $k=0,\ldots,n$, and vector $\bxi \in \ssB_{n,k}$,
let $V^*_k(\bxi\,;\,f)$ denote the maximum
of $f(x_1,\ldots,x_k,\,\xi_{k+1},\ldots,\xi_n)$ over $(x_1,\ldots,x_k)\in \{0,1\}^k$.  
Thus, $V^*_n(\bxi\,;\,f)=V^*_n(\bzero;\,f)$
is equal to the maximum of $f$.
By definition,
\[ V^*_0(\bxi\,;\,f) = f(\bxi)~,\]
and for $k=1,\ldots,n$,
\[
V^*_k(\bxi\,;\,f)
~=~
\max\big\{
V^*_{k-1}(\bxi\,;\,f)~,~
V^*_{k-1}(\bxi+\be^k;\,f)
\big\}~,
\]
where $\be^k\in \{0,1\}^n$ denotes the unit vector with $1$ in the $k$th position.
Let $V_k(\bxi\,;\,f)$ be any function whose domain is the domain of 
$V^*_k$, and such that
for every $\bxi\in \{0,1\}^n$,
$V_0(\bxi\,;\,f)= f(\bxi)$.
We define
\[
\Phi(V,V^*) = \EE_{f\in\ssF}
\bigg[ 
\sum_{\bxismall\in B^n}
\big|V^*_n(\bxi;\,f) - 
  V_n(\bxi;\,f)\big|
\bigg]  
\]
and
\[
\Psi 
(V) = \EE_{f\in\ssF}\bigg[
\sum_{k=1}^n \sum_{\bxismall \in B_{n,k}}  
\big| 
\max
\big\{
V_{k-1}(\bxi\,;\,f) ~,~ 
V_{k-1}(\bxi+\be^k;\,f)
\big\} 
 - V_{k}(\bxi\,;\,f)
\big|~
\bigg]~.
\]
Denote
\eq \label{eqqn:3}
\delta 
(k\,;f,\,\bxi) = 
\max\big\{
V_{k-1}(\bxi\,;\,f) ~,~ 
V_{k-1}(\bxi + \be^k\,;\,f)\big\} 
- 
V_{k}(\bxi\,;\,f)
\eeq
and
\[
\Delta(k\,;f,\bxi)
= V^*_k(\bxi\,;\,f)
- V_k(\bxi\,;\,f)
\]
For convenience, denote
\[ \overline\delta(k\,;f,\bxi) = \big|\delta(k\,;f,\bxi)\big|\]
and
\[ \overline\Delta(k\,;f,\bxi) = \big|\Delta(k\,;f,\bxi)\big| ~.\]
We will rely on the following simple proposition:
\begin{prop} \label{ineq}
For every set 
$\{A,B,a,b\} \subset \bbR$, 
the following inequality holds: 
\[
\big|\max\{A,B\} 
   - \max\{a,b\}\big| 
\le |A-a| + |B-b| ~.\]
\end{prop}

\begin{prop}
For every $n\in \bbN$, $k=1,\ldots,n$, $f$, $V$, and $\bxi\in \ssB_{n,k}$,
the following inequality holds:
\[
\overline\Delta
\big(k\,;\,f,\,\bxi\big)
~\le 
\overline\Delta \big(k-1\,;\,f,\,\bxi\big) 
+ 
\overline\Delta \big(k-1\,;\,f,\,\bxi ~+~ \be^k \big)
+ \overline\delta\big(k\,;\,f,\,\bxi\big)~.
\]
\end{prop}
\begin{proof}
We have
\[\eqalign{
\overline\Delta(&k\,;f,\bxi) 
= \big|
V^*_k(\bxi;\,f)
- V_k(\bxi;\,f)
\big|  \cr
\le &\
\big|
\max\big\{
V^*_{k-1}(\bxi;\,f) ~,~ V^*_{k-1}(\bxi+\be^k;\,f)
\big\}
- \max\big\{
V_{k-1}(\bxi;\,f) ~,~ 
 V_{k-1}(\bxi+\be^k;\,f)
\big\}
\big|\cr
&  +
\big|
\max\big\{
V_{k-1}(\bxi;\,f) ~,~ V_{k-1}(\bxi+\be^k;\,f)
\big\}
- V_k(\bxi;\,f)
\big| \cr
\le &\
\big|
V^*_{k-1}(\bxi;\,f) - V_{k-1}(\bxi;\,f)
\big|  
~+ \big|
V^*_{k-1}(\bxi+\be^k;\,f) -
 V_{k-1}(\bxi+\be^k;\,f)
\big|
+
\overline\delta\big(k\,;\,f,\,\bxi\big)
 \cr
=&\ 
\overline\Delta \big(k-1\,;\,f,\,\bxi\big) + 
\overline\Delta \big(k-1\,;\,f,\,\bxi + \be^k \big)
+ \overline\delta\big(k\,;\,f,\,\bxi\big)~.\qedhere
}\]
\end{proof}

\begin{prop}
For every $n\in\bbN$, $k=1,\ldots,n$, $f$, $V$, and $\betta\in \ssB_{n,k}$
\[
\overline\Delta(k\,;f,\,\betta) 
~\le~ \sum_{\ell=1}^k ~
\sum_{\bxismall \in \Xi_\ell(\bettasmall,\,k)}
\overline\delta\big(\ell\,;\,f,\,\bxi\big)
\]
\end{prop}
\begin{proof}
By induction on $k$, 
\[
\eqalign{
\overline\Delta&(k\,;f,\,\betta) 
\le  
\overline\Delta \big(k-1\,; \,f,\,\betta\big) 
~~+ ~~
\overline\Delta \big(k-1\,;\,f,\,\betta + \be^k \big)
+ \overline\delta\big(k\,;\,f,\,\betta\big)
\cr
\le &\
\sum_{\ell=1}^{k-1}~
\sum_{\bxismall \in \Xi_{\ell}(\bettasmall,\,k-1)}
\overline\delta(\ell\,;\,f,\,\bxi)
~~+~~  \sum_{\ell=1}^{k-1} ~
\sum_{\bxismall \in \Xi_{\ell}(\bettasmall + \besmall^k,\,k-1)}
\overline\delta(\ell\,;\,f,\,\bxi)
~+~ 
\overline\delta\big(k\,;\,f,\,\betta\big) \cr
= &\
\sum_{\ell=1}^{k-1} ~
\sum_{\bxismall\in \Xi_{\ell}(\bettasmall, \,k)} 
\overline\delta\big(\ell\,;\,f,\,\bxi\big) 
~+~ 
\overline\delta\big(k\,;\,f,\,\betta\big) 
=
\sum_{\ell=1}^{k} ~
\sum_{\bxismall\in \Xi_{\ell}(\bettasmall,\,k)} 
\overline\delta\big(\ell\,;\,f,\,\bxi\big) ~. \qedhere\cr
}
\]
\end{proof}
Below we consider the possibility of generating approximate value-functions $V$ of increasing accuracy by a training a neural net model.  
The function
$\Psi(V)$ will assist in obtaining
better functions $V$.

\section{\hskip -16pt. Parameterized approximate value functions}
We will consider
{\em value mappings} 
$V(f;\btheta)$, parameterized by a vector $\btheta$, which approximate
the optimal value function
$V^*(f)$ for $f\in\ssF$.
Denote by $\Theta$ the set of possible parameter vectors $\btheta$.

Given $V(\bul\,;\, \btheta)$ for a certain $\btheta\in\Theta$, a candidate solution 
$\bx=\bx(f;\,\btheta)$  can be generated for every 
$f\in \ssF_n$ by fixing the values of $x_1,\ldots,x_n$, successively, guided by the values 
$V(f^\bot;\,\btheta)$ 
of ``residual'' sub-instances $f^\bot$ of $f$.

For every $f\in \ssF_n$,
$k=1,\ldots,n$, and values
$\bxi=(\xi_{k+1},\ldots,\xi_n)\in\{0,1\}^{n-k}$,
denote by $f_{\bxismall}^\bot$
the sub-instance of $f$ that calls for maximizing
$f(x_1,\ldots,x_{k},\,
\xi_{j+1},\ldots,\xi_n)$
over $(x_1,\ldots,x_k)\in \{0,1\}^k$.  
Recall that
$V^*_k(\bxi\,;\,f)$ represents
the maximum of 
$f_{\bxismall}^\bot$.
By $\{0,1\}^0$ we mean the null string $\emptyset$. Thus, $f^\bot_\emptyset$ is the same as $f$.
Suppose $V(\bxi\,;\,f\,; \,\btheta)$
is an estimate for 
$V^*_k(\bxi\,;\,f)$.
\begin{theorem}
Given are an instance $f\in \ssF_n$ of the optimization problem
and a parameter vector $\btheta\in \Theta$. Suppose the following conditions hold: 
\begin{enumerate}
\itemi
For every $\bxi\in \{0,1\}^n$,
$V(\bxi\,;f\,;\,\btheta)=f(\bxi)$.
\itemii
For $k=1,\ldots,n$, and
$\bxi \in \{0,1\}^{n-k}$,
\[ V(\bxi;f;\btheta)
= \max
\big\{
V(\,(0,\bxi);f;\btheta)
~,~
V(\,(1,\bxi);\,f;\btheta)
\big\}
~.\]
\end{enumerate}
Under these conditions, for $k=1,\ldots,n$ and $\bxi\in \{0,1\}^{n-k}$, $V(\bxi\,;\,f\,;\btheta)$ is equal to the optimal value of the sub-instance $f^\bot_{\bxismall}$.
\end{theorem}
\begin{proof}
The proof is by induction on $k$. 

First, for $k=0$, by condition $\ist$, 
$\bxi\in \{0,1\}^n$, $V(\bxi\,;\,f\,;\,\btheta)=f(\bxi)$
so it is optimal.

Second, for $k>0$, let $\bxi\in\{0,1\}^{n-k}$.
By the induction hypothesis, since
$(1,\bxi) \in\{0,1\}^{n-(k-1})$ 
and 
$(0,\bxi) \in\{0,1\}^{n-(k-1)}$,
then
$V((0,\bxi)\,;\,f;\,\btheta)$
is equal to the optimal value of
$f^\bot_{(0,\bxismall)}$,
and
$V((1,\bxi)\,;\,f;\,\btheta)$
is equal to the optimal value of
$f^\bot_{(1,\bxismall)}$.
However, the optimal value of 
$f^\bot_{\bxismall}$ is obviously equal to the maximum of the optimal values of $f^\bot_{(0,\bxismall)}$
and $f^\bot_{(1,\bxismall)}$, hence it is equal to $V(\bxi\,;\,f;\,\btheta)$.
\end{proof}
\subsection{An iteration of stochastic gradient}
Recall that
\[ 
\Psi_n(V) = \EE_{f\in\ssF}\bigg[
\sum_{k=1}^n 
\sum_{\bxismall \in B_{n,k}} \overline\delta(k\,;\,f\,;\,\bxi)
\bigg]~.
\]
When the valuation of $V$ is carried out by a parameterized mechanism, i.e.,
$V(f)=V(f\,;\btheta)$, the values of
$\delta(k\,;\,f\,;\,\bxi)$ and
$\Psi_n(V)$ depend on $\btheta$,
and we consider the following function:
\eq  \label{eq:10}
\varphi(\btheta) = 
\ \Psi_n(V(\bul;\btheta))
= \EE_{f\in \ssF_n}
\bigg[
\sum_{k=1}^n 
\sum_{\bxismall \in B_{n,k}} \overline\delta(k\,;\,f\,;\,\bxi\,;\btheta)
\bigg]  
= \sum_{f\in \ssF_n} 
\sum_{k=1}^n 
\sum_{\bxismall \in B_{n,k}} 
p(f)\cdot\overline\delta(k\,;\,f\,;\,\bxi\,;\btheta)
~. 
\eeq
Obviously, $\varphi(\btheta)$ is a sum of many functions of $\btheta$, which are differentiable except for a set of measure zero, namely, at points where the maximum in some term is not differentiable.
See Remark \ref{rem:6} below.
Hence, the gradient of $\phi(\btheta)$ is equal to the sum of the gradients of all those functions, i.e.,
\[
\nabla_{\bthetasmall}\,
\varphi(\btheta)
= \sum_{f\in \ssF_n} 
\sum_{k=1}^n 
\sum_{\bxismall \in B_{n,k}} 
p(f)\cdot
\sgn(\delta(k\,;\,f\,;\,\bxi\,;\btheta))\cdot
\nabla_{\bthetasmall}\,
\delta(k\,;\,f\,;\,\bxi\,;\btheta)
~. 
\]
Recall that there is underlying probability distribution $p$ over the set $\ssF$ of instances.  
We can work either with a fixed dimension of instances or with a conditional distribution based on $p$, given a maximum dimension on instances
To carry out a stochastic-gradient step
for $f(\btheta)$, we take a sample of size $r$ of sub-instances as follows. We take a sample of size $r$ from a uniform distribution over all possible 
sub-instances of instances in $\ssF$, where each sub-instance $f^\bot_{\bxismall}$ is defined by a pair 
$(f,\, \bxi)$,  
so that 
if $f\in \ssF_n$, then
$\bxi \in B^{n-k}$ for some 
$k \le n$.  
The sub-instance contributes a term
$g(\btheta)= g(\btheta\,;\,f,\,\bxi)$
to the stochastic gradient, 
weighted by the probability $p(f)$
assigned to the instance $f$
in the underlying distribution $p$.
For example, if sub-instance $(f,\bxi)$ is sampled, then the corresponding term that appears in the sum in (\ref{eq:10}) is included in the stochastic gradient, multiplied by $p(f)$.
\begin{remark} \label{rem:6} \rm
Note that for
$k$, $f$ and $\bxi$, the function of $\btheta$ defined by
$\delta(k;f;\bxi;\btheta)$
as in (\ref{eqqn:3}), i.e.,
\[
\delta(k\,;f,\,\bxi,\btheta) = 
 \max\big\{V_{k-1}(\bxi;f;\,\btheta) \,,\,
V_{k-1}(\bxi + \be^k;f;\,\btheta)\big\} 
- 
V_{k}(\bxi;f;\btheta)
\]
may not be differentiable
due to the $\max$ operation.
A common way to get a smooth approximation is the following:
\[ 
\max(x,y) \approx
\frac{
  x\, e^{\alpha x}
+ y\, e^{\alpha y}
}
{e^{\alpha x} + e^{\alpha y}}~,
\]
where $\alpha >0$ is an optional parameter that controls the approximation.

Similarly, the absolute-value function
\[ |x| = \max\{x,\,-x\}
\] 
can be approximated by
\[a(x) = x \cdot \frac{e^{\alpha\,x} 
       - e^{-\alpha\,x}}
       {e^{\alpha\,x} 
       + e^{-\alpha\,x}}
       \]
or, more simply, by
\[ b(x) = \sqrt{x^2 + \alpha^2} ~.\]
\end{remark}

\section{\hskip -16pt. Two approaches to the Knapsack problem}

The standard Knapsack packing problem is the following:
\eq \label{eqn:23}
\mbox{Maximize }~  \sum_{j=1}^n c_j\,x_j
 \mbox{~~subject to }  ~ 
\sum_{j=1}^n a_j\,x_j \le b ~~~\mbox{~and~}
 x_1,\ldots,x_n\in \{0,1\}
\eeq
where $a_j > 0$, $j=1,\ldots n$.
An instance is feasible if and only if $b\ge 0$. 
However, depending on how the dynamic programming algorithm is set up, infeasible  sub-instances may arise. A a simple way to avoid that is as follows.
\subsection{Avoiding infeasible sub-instances}\label{sec3.1}
For $a_1,\ldots,a_k,b\ge 0$ and any
$c_1,\ldots,c_k\in \bbR$, denote by
$V^*((c_1,\ldots,c_k),\,(a_1,\ldots,a_k),\,b)$ the maximum of 
$\sum_{j=1}^k c_j\,x_j$ subject to
$\sum_{j=1}^k a_j\,x_j \le b$ ($x_1,\ldots,x_k\in \{0,1\}$.
For stating a  dynamic-programming equation,
denote $\bc = (c_1,\ldots,c_{k-1})$ and
$\ba = (a_1,\ldots,a_{k-1})$. The equation is the following:
\[ 
V^*((\bc,c_k),\,(\ba,a_k),\,b) =
\begin{cases}
    \max\{c_k + V^*(\bc,\,\ba,\, b - a_k)
    ~,~ V^*(\bc,\,\ba,\, b)\} 
   & \mbox{if $b\ge a_k$}  \\
    ~~~~~ V^*(\bc,\,\ba,\, b)   
   & \mbox{if $b < a_k$}
\end{cases}  
\]

For brevity, denote an instance of the problem (\ref{eqn:23})
by
$ \bi = (\bc,\,\ba,\,b)$,
and denote the set of such instances of dimension $n$ by $\ssI_n$.
 For $k=1,\ldots,n$ and $\bxi=(0,\ldots,0,\,\xi_{k+1},\ldots,\xi_n)\in B^{n}$, such that 
 $\sum_{j=k+1}^n a_j\xi_j \le b$,
denote by $V^*(k\,;\,\bi,\,\bxi)$ the maximum objective-function value of the following problem:
\[
\mbox{Maximize }~ 
\sum_{j=1}^k c_j\,x_j ~~
\mbox{subject to }~
\sum_{j=1}^k a_j\,x_j 
~\le~ b - \sum_{j=k+1}^n a_j\xi_j
\mbox{~~and~~} x_1,\ldots,x_k\in \{0,1\}~. 
\]
Let $\be^k\in B^n$ denote the unit vector with $1$ in the $k$th position.
The dynamic-programming equation for the latter problem is the following:
\[ 
V^*(k;\,\bi,\bxi)
=
\begin{cases}
\max\big\{V^*(k-1; \bi,\bxi)
    \,,\, c_{k} + 
    V^*(k-1; \bi,\bxi)
    \big\}&
    \mbox{if~} b\ge a_k+ \sum_{j=k+1}^n a_j\xi_j \\
    ~~~V^*(k-1; \bi,\bxi)  ~~~~~~~~~~~~~~~~~~~~~~~
   & \mbox{otherwise}
    \end{cases}
\]
where $V^*(0\,;\,\bul)=0$.
If $V(k\,;\,\bi,\,\bxi)$ is a function whose domain is the domain of $V^*$,
and $V(0\,;\,\bul)=0$, then we define
\[ \Phi_n(V,\,V^*)
=\EE_{ \bismall\in \ssI_n}
\bigg[~\big|
V(n\,;\bi,\,\bzero) -
V^*(n\,;\bi,\,\bzero) 
\big|~
\bigg]~.
\]
For $n\in \bbN$ and $k=1,\ldots,n$, denote
\[ \ssB_{n,k}
= \left\{\bxi \in \{0,1\}^n\,:\, \xi_1=\cdots=\xi_k=0\right\} ~.\]
Also, for $\ell=1,\ldots,k$
and $\betta \in \ssB_{n,k}$,
denote
\[ 
\Xi_\ell(\betta)
= 
\Xi_\ell(\betta,\,k) 
= 
\big\{ \bxi \in \ssB_{n,\ell}~:~
\xi_j = \eta_j,~ j=k+1,\ldots,n \big\} ~.
\]
Note that
$ \Xi_k(\betta) =  \big\{\betta\big\}$.
\begin{definition}\rm
If 
$\sum_{j=k+1}^n a_j\,\xi_j \le b $, 
define $\delta(k\,;\bi,\,\bxi)$ as follows.
\begin{enumerate}
\itemi
if $a_k+\sum_{j=k+1}^n a_j\,\xi_j \le b$,
define 
\[ 
\delta(k\,;\bi,\,\bxi) = 
\max\big\{V(k-1\,;\,\bi,\,\bxi)
    \,,\, c_{k} + 
    V(k-1\,;\,\bi,\,\bxi+\be^k)
    \big\} 
   - V(k\,;\,\bi,\,\bxi) 
~;
\]    
\itemii
if 
$a_k+ \sum_{j=k+1}^n a_j\,\xi_j
> b$,
define
\[ \delta(k\,;\bi,\,\bxi) =
V(k-1\,;\,\bi,\,\bxi)
- V(k\,;\,\bi,\,\bxi) 
\]
\end{enumerate}
and, for convenience, denote
$\overline\delta(k\,;\bi,\,\bxi) =
\big|\delta(k\,;\bi,\,\bxi)\big|$.
Also, define
\[ \Psi_n(V)
=\EE_{\bismall\in \ssI_n}
\bigg[\sum_{k=1}^n
\sum_{\bxismall \in \ssB_{n,k}} \bigg\{
\overline\delta(k\,;\bi,\,\bxi) ~:~
\sum_{j=k+1}^n a_j\,\xi_j \le b 
\bigg\}
\bigg]~.
\]
\end{definition}
\begin{definition}\rm
If 
$ \sum_{j=k+1}^n a_j\,\xi_j \le b$,
define
\[ \Delta(k\,;\bi,\,\bxi)
= V^*(k\,;\bi,\,\bxi) -
 V(k\,;\bi,\,\bxi) ~,
\]
and, for convenience, denote
$\overline\Delta(k\,;\bi,\,\bxi)=\big|\Delta(k\,;\bi,\,\bxi)\big|$.
\end{definition}
We will rely on  Proposition \ref{ineq}.
\begin{prop} \label{prop3.4}
The following inequality holds for every $n\in\bbN$, an instance $\bi\in\ssI_n$,
$k=1,\ldots,n$, a value-function $V$,
and $\betta \in \ssB_{n,k}$such that
\/ $\sum_{j=k+1}^n a_j\, \xi_j \le b$:
\[  
\overline\Delta(k\,;\,\bi,\,\betta) 
~\le ~
\sum_{\ell=1}^k ~
\sum_{\bxismall \in \Xi_\ell(\bettasmall,\,k)} 
\bigg\{
\overline\delta(\ell\,;\,\bi,\,\bxi)
~:~ \sum_{j=\ell+1}^n a_j\,\eta_j \le b
\bigg\}~.
\]
\end{prop}
\begin{proof}
The proof is by induction on $k$.
First, for $\betta$ such that
$
b-a_k < \sum_{j=k+1}^n a_j\,\eta_j \le b
$,
we have
\[  
\eqalign{
\overline\Delta(k\,;\,\bi,\,\betta)  
\le &\
\big|V^*(k-1\,;\,\bi,\,\betta) -   
V(k-1\,;\,\bi,\,\betta)\big|
~+~
\big|V(k-1\,;\,\bi,\,\betta) -   
V(k\,;\,\bi,\,\betta)\big|
\cr
=&\
\overline\Delta(k-1\,;\,\bi,\,\betta) 
 +
\overline\Delta(k\,;\,\bi,\,\betta)
 ~.\cr
}\]
Second,
for $\betta\in \ssB_{n,k}$ such that
$ a_k +\sum_{j=k+1}^n a_j\eta_j \le b$, 
we rely on Proposition \ref{ineq}
and obtain

\[\overline\Delta(k\,;\bi,\betta) \le
 \overline\Delta(k-1\,;\,\bi,\,\betta)
+\overline\Delta(k-1\,;\,\bi,\,\betta+\be^k)
+\overline\delta(k\,;\,\bi,\,\betta) ~.
\]
By the induction hypothesis,
for $\betta\in \ssB_{n,k-1}$ such that
$\sum_{j=k}^n a_j\eta_j \le b$,
\[  
\overline\Delta(k-1;
\bi,\,\betta) 
~~\le ~~
\sum_{\ell=1}^{k-1} ~
\sum_{\bxismall \in \Xi_\ell(\bettasmall,\,k-1)} 
\bigg\{
\overline\delta(\ell\,;\,\bi,\,\bxi)
:
\sum_{j=\ell+1}^n a_j\,\xi_j \le b\bigg\}~,
\]
and if $b - a_k < \sum_{j=k+1}a_j\,\xi_j \le b$, then
\[ 
\overline\Delta(k-1\,;\,
\bi,\,\betta+\be^k) 
~~\le~~
\sum_{\ell=1}^{k-1} ~
\sum_{\bxismall \in \Xi_\ell(\bettasmall+\besmall^k,\,k-1)} 
\bigg\{
\overline\delta(\ell\,;\,\bi,\,\bxi)
:
\sum_{j=\ell+1}^n a_j\,\xi_j \le b\bigg\}~,
\]
hence, for $\betta\in\ssB_{n,k}$ such that
$\sum_{j=k+1}^n a_j, \eta_j \le b$,
\eq \label{eqn:34}
\overline\Delta(k;\,\bi,\betta) 
\le
\sum_{\ell=1}^{k} ~
\sum_{\bxismall\in \Xi_\ell(\bettasmall,\, k)}
\bigg\{
\overline\delta(\ell\,;\,\bi,\bxi)
:
 \sum_{j=\ell+1}^n a_j\,\xi_j\le b
\bigg\} ~.\qedhere
\eeq
\end{proof}
By taking the expectations of both sides of (\ref{eqn:34}) over instances $\bi$, we get the following inequality:
\begin{theorem} \label{theo}
For every $n$ and every value function $V$,
\[
\Phi_n(V,V^*) \le \Psi_n(V) ~.
\]
\end{theorem}
\subsection{Formulation with guaranteed feasibility}
The problem stated in (\ref{eqn:23})
is equivalent to the following problem 
with one artificial variable:
\eq \label{eqn:36}
\mbox{Maximize }~ 
\sum_{j=1}^n c_j\,x_j - \sum_{j=1}^n c_j \cdot x_0 ~~
\mbox{s.t.~}
\sum_{j=1}^n a_j\,x_j -
\sum_{j=1}^n a_j\cdot x_0 
\le b \mbox{~~and~~}
~ x_0,x_1,\ldots,x_n\in \{0,1\}~,
\eeq 
Obviously, the latter is feasible even 
if $b<0$.  
Denote an instance by
$ \bi = (\bc,\,\ba,\,b)$,
and the set of such instances of dimension $n$ by $\ssI_n$. 
For $k=1,\ldots,n$ and $\bxi=(0,\ldots,0,\,\xi_{k+1},\ldots,\xi_n)\in B^{n}$,
denote by $V^*(k\,;\,\bi,\,\bxi)$ the maximum objective-function value of the following problem:
\[ 
\mbox{Maximize} 
\sum_{j=1}^k c_jx_j - \sum_{j=1}^n c_jx_0~
\mbox{s.t.}~
\sum_{j=1}^k a_jx_j 
- \sum_{j=1}^n a_jx_0
\le b - \sum_{j=k+1}^n a_j\xi_j
\mbox{, }
x_0, x_1,\ldots,x_k\in \{0,1\}~. 
\]
The optimality equation of the latter is the following:
\[  
V^*(k;\,\bi,\,\bxi) 
=
\max\big\{V^*(k-1;\bi,\bxi)
   \, ,\, c_{k} + 
    V^*(k-1;\bi,\,\bxi+\be^k)
    \big\} 
\]
and the functions $\Phi$ and $\Psi$ (appropriately defined with restriction to feasible instances) satisfy the inequality of Theorem \ref{theo}.

\begin{remark}\rm
Consider the formulation of the Knapsack problem
as in (\ref{eqn:36}), and assume $b$ and the $a_j$s are integers. Thus, if $\sum_{j=1}^na_j\, x_j > b$, then $\sum_{j=1}^na_j\, x_j - b \ge 1$. 
We may define the value 
$V^*(\bc,\,\ba,\,b)$
of an instance $(\bc,\,\ba,\,b)$
with the following set of functions, 
for $\bx\in \{0,1\}^n$:
\vskip -12pt
\[ 
\phi(\bc,\,\ba,\,b)
= \phi(\bc,\,\ba,\,b\,;\,
\bx))  
=\bc^\top \bx - \sum_{j=1}^n c_j\cdot
\max\{0,\,\ba^\top \bx - b
\}~.
\]
\vskip -6pt
Thus,
\[ V^*(\bc,\,\ba,\,b)
= \max\big\{
\phi(\bc,\,\ba,\,b\,;\,
\bx))~:~\bx \in \{0,1\}^n 
\big\}
~,\]
hence $V^*$ is continuous.
\end{remark}
\section{\hskip -16pt. Examples}
\subsection{Maximum Weighted Satisfiability}
\begin{definition} \rm
The {\em Maximum-Weighted-Satisfiability}
problem is defined as follows.
Given $m$ Boolean functions in $n$ variables, $\beta_i(x_1,\ldots,x_n)$
and rational coefficients $c_i$,
$i=1,\ldots,m$,
assign truth-values for $x_1,\ldots,x_n\in\{0,1\}$ so as to maximize the function
$\sum_{i=1}^m c_i\, \beta_i\,(x_1,\ldots,x_n)$.
\end{definition}
Given functions $\beta_i$s and coefficients $c_i$s, 
for $k=1,\ldots,n$ and $(\xi_{k+1},\ldots,\xi_n)
\in \{0,1\}^{n-k}$,
denote by 
$V^*_k(\xi_{k+1},\ldots,\xi_n)$
the maximum (over $(x_1,\ldots,x_k)\in 
\{0,1\}^k$ of  the sum\\
$\sum_{i=1}^m c_i\, \beta_i\,(x_1,\ldots,x_k,\xi_{k+1},\ldots,\xi_n)$.
The optimality condition is the following:
\[ 
V_k(\xi_{k+1},\ldots,\xi_n)
= 
\max\big\{
V_{k-1}(0,\,\xi_{k+1},\ldots,\xi_n)
,~
V_{k-1}(1,\,\xi_{k+1},\ldots,\xi_n)
\big\}~.
\]
Thus, Theorem \ref{theo}
can be applied directly.

\subsection{Maximum weighted independent set}
\begin{definition}\rm
The {\em Maximum-Weighted-Independent-Set}
problem is defined as follows.
Given a weighted graph $(G,\bw)$, i.e., a graph $G$ on $n$ nodes with weights
$\bw =(w_1,\ldots,w_n) \in \bbR^n$, find a set $S\subset N=\{1,\ldots,n\}$ of independent nods, (i.e., no two nodes in $S$ are connected by an arc) so as to maximize
$\sum_{i\in S} w_i$.
\end{definition}
\begin{definition}\rm
Given a weighted  graph $(G,\bw)$,
\vskip - 36pt
\begin{enumerate}
\itemi
for every node $i\in N$, denote by 
$N(i)$ the set of neighbors of $i$, i.e.,
$i\in N(i)$ and $j\in N(i)$ if 
$(i,j)$ is an arc;
\itemii
for every subset of nodes $S\subset N$,
denote by $G\setminus S$ the graph obtained by $G$ by dropping the nodes in $S$ as well as all the arcs incident on any $i\in S$; 
denote by $\bw \setminus S$ the vector obtained from $\bw$ by dropping all the entries corresponding to any $i\in S$.
\end{enumerate}
\vskip -12pt
\end{definition}
Denote by $V^*(G,\bw)$ the maximum weight of an independent set of nodes in $G$. 
Let the nodes of $G$ be $i_1<\cdots<i_k$ (so
$\bw=(w_{i_1},\ldots,w_{i_k})$).
The optimality equation for $V^*$ is the following:
\eq  \label{eqn:26}
V^*(G,\bw) = 
\max\big\{
V^*(G\setminus\{i_k\}
,\,\bw \setminus\{i_k\})~,~
w_{i_k} + V^*(G\setminus N(i_k)
,\,
\bw \setminus N(i_k))\big\}~.
\eeq
Denote by $V(G,\bw)$ any function with the same domain as $V^*$.
\begin{prop}\label{prop:3.1ind}
The following inequality holds for every
weighted graph $(G,\bw)$ with $k$ nodes and value-function $V$:
\eq \label{eqn:32}
\eqalign{
\big|
V^*(&G,\,\bw) -  
V(G,\,\bw)\big|  
\le 
\big|
V^*(G\setminus \{i_k\},\,\bw\setminus \{i_k\})
- V(G\setminus \{i_k\},\,\bw\setminus \{i_k\})
\big|\cr
&+\big|
V^*(G\setminus N(i_k),\,\bw\setminus N(i_k))
- V(G\setminus N(i_k),\,\bw\setminus N(i_k))
\big|\cr
&+\big|\max\big\{
 V(G\setminus\{i_k\},\,\bw\setminus\{i_k\})~,~
w_{i_k} + V(G\setminus N(i_k),\,\bw\setminus N(i_k))
 \big\}  
 - V(G,\,\bw)\big|~.
}
\eeq
\end{prop}
\begin{definition}\rm
A {\em weighted sub-graph} of a weighted graph $(G,\bw)$ is a pair $(H,\bw_H)$ where $H$ is a graph obtained from $G$ by dropping a subset of its nodes as well ass all the arcs incident upon at least one of the dropped nodes, and $\bw_H$ is the sub-vector of $\bw$ consisting of the entries of $\bw$ that correspond to nodes in $H$. 
We denote by $i_H$ the node of the highest index that belongs to $H$.
\end{definition}
The inequality (\ref{eqn:32}) can be applied iteratively and therefore implies the following:
\begin{prop}
The following inequality holds for every
natural $k$, a weighted graph $(G,\,\bw)$ with nodes $i_1<\cdots<i_k$, and a value-function $V$, there exist integral coefficients $\alpha_H$ such that
\[ 
\eqalign{
\big|V^*&(G,\,\bw) 
-V(G,\,\bw)\big|  
\le \sum_{H\subseteq G}\alpha_H \, \times\, \cr
& \bigg|
\max
\big\{ 
V(H\setminus\{i_H\},\,\bw_H\setminus \{i_H\}) ~,~ 
w_{i_H} + 
V(H\setminus N(i_H),\, \bw_H\setminus N(i_H))
\big\}  
 - V(H,\,\bw_H)
\bigg|~.
}
\]
\end{prop}
The coefficient $\alpha_H$ represents the number of times the sub-graph $H$ occurs in the recursive process implied by (\ref{eqn:26})
for calculating 
$V^*(G,\,\bw)$ of the given weighted graph.
An explicit expression can be derived with the help of a vector 
$\bxi=(\xi_{k+1},\ldots,\xi_n)\in \{0,1\}^{n-k}$
as follows.
Suppose node $i$, $i>k$, has already been picked to be in the independent set if and only if $\xi_i=1$.
Denote by $V^*(G,\,\bw,\,\bxi)$
the maximum weight of
an independent set $S$ of nodes, subject to the constraint that for $i=k+1,\ldots,n$, $i\in S$ if and only if $\xi_i=1$. In this case, we have the following, already familiar, optimality equation:
\[
V^*(G,\,\bw,\,\bxi)
= \max\big\{
V^*(G,\,\bw,\,(0,\bxi)),
~ V^*(G,\,\bw,\,(1,\bxi))
\]
if node $k$ is not adjacent to any
$i$ with $\xi_i=1$;
otherwise,
\[
V^*(G,\,\bw,\,\bxi)
= 
V^*(G,\,\bw,\,(0,\bxi))~.
\]
If the distribution over weighted graphs is over a fixed graph, then the coefficients $\alpha_H$ could be determined during preprocessing.

\subsection{The Maximum-Cut Problem}
\begin{definition} \rm
The Maximum-Cut (Max-Cut) problem is defined
over a graph with $n$ nodes and a reward matrix
$\bR = (R_{ij}) \in \bbR^{n\times n}$.
It calls for finding $\bx \in \{0,1\}^n$
so as to maximize
\[ 
\sum_{i=1}^n \sum_{j=1}^n
  R_{ij} x_i (1-x_j)~. \]
\end{definition}
For $\bxi=(0,\ldots,0,\xi_{k+1},\ldots,\xi_{n})\in \{0,1\}^{n}$, 
denote by 
$V^*(k\,;\,\bR,\,\bxi)$ the maximum of the function
\[ 
f(x_1,\ldots,x_k) 
= \sum_{i=1}^k\sum_{j=1}^k 
           R_{ij}\,x_i\, (1-x_j)
           +
\sum_{i=1}^k \bigg(
(1-x_i)
\sum_{j=k+1 \atop \xi_j=1}^n  
      R_{ij}
+x_i\sum_{j=k+1 \atop \xi_j=0}^n
      R_{ij} 
 \bigg)~.
 \]
Denote by $\be^k$ the unit $n$-vector with $1$ in the $k$th position.
The optimality equation is the following:
\[
V^*(k;\bR,\bxi) 
= 
\max\bigg\{ 
\sum_{j=k+1 \atop \xi_j=1}^n R_{kj}
+
 V^*(k-1; \bR,\bxi)
~,~  
\sum_{j=k+1 \atop \xi_j=0}^n R_{kj}
+
V^*(k-1;\bR, \bxi + \be^k )
\bigg\}~.
\]



\begin{thebibliography}{1}

\bibitem{bengio2021machine}
Y.~Bengio, A.~Lodi, and A.~Prouvost.
\newblock Machine learning for combinatorial optimization: a methodological
  tour d?horizon.
\newblock {\em European Journal of Operational Research}, 290(2):405--421,
  2021.

\bibitem{bresson2021transformer}
X.~Bresson and T.~Laurent.
\newblock The transformer network for the traveling salesman problem.
\newblock {\em arXiv preprint arXiv:2103.03012}, 2021.

\bibitem{kool2018attention}
W.~Kool, H.~Van~Hoof, and M.~Welling.
\newblock Attention, learn to solve routing problems!
\newblock {\em arXiv preprint arXiv:1803.08475}, 2018.

\bibitem{mnih2015human}
V.~Mnih, K.~Kavukcuoglu, D.~Silver, A.~A. Rusu, J.~Veness, M.~G. Bellemare,
  A.~Graves, M.~Riedmiller, A.~K. Fidjeland, G.~Ostrovski, et~al.
\newblock Human-level control through deep reinforcement learning.
\newblock {\em nature}, 518(7540):529--533, 2015.

\bibitem{zhang1995value}
W.~Zhang and T.~G. Dietterich.
\newblock Value function approximations and job-shop scheduling.
\newblock In {\em Pre-prints of Workshop on Value Function Approximation in
  Reinforcement Learning at ICML-95}. Citeseer, 1995.

\end{thebibliography}

\end{document}